\documentclass[letterpaper, 9 pt, journal, twoside]{IEEEtran} 
                                                          
\IEEEoverridecommandlockouts                              

\usepackage[utf8]{inputenc}
\usepackage[cmex10]{amsmath}
\usepackage{amsfonts}
\usepackage{amssymb}
\usepackage{cite}

\newtheorem{theorem}{\bf Theorem}

\newtheorem{corollary}{\bf Corollary}
\newtheorem{definition}{\bf Definition}
\newtheorem{lemma}{\bf Lemma}
\newtheorem{invariant}{\bf Invariant}

\usepackage{psfrag}
\usepackage[dvips]{graphicx}
\graphicspath{{figs/}}

\usepackage[tight]{subfigure}

\title{Capturing an Omnidirectional Evader in Convex Environments using a Differential Drive Robot}

\author{Ubaldo Ruiz$^1$ and Volkan Isler$^2$ 
	\thanks{$^1$Ubaldo Ruiz is a CONACYT Research Fellow. He is with the Department of Computer Science, CICESE, Baja California, Mexico {\tt\small uruiz@cicese.mx}.}
	\thanks{$^2$Volkan Isler is with the Department of Computer Science and Engineering, University of Minnesota, Minneapolis, MN, USA {\tt\small isler@cs.umn.edu}}
}

\begin{document}

\maketitle


\begin{abstract}
We study the problem of capturing an Omnidirectional
Evader in convex environments using a Differential
Drive Robot (DDR). The DDR wins the game if at any time instant it captures (collides with) the evader. 
The evader wins if it can avoid capture forever. Both players are unit disks with the same maximum (bounded) speed, but the DDR can
only change its motion direction at a bounded rate.  We show that despite this limitation, the DDR can capture the evader.
\end{abstract}

\begin{IEEEkeywords}
Motion and Path Planning,
Surveillance Systems,
Nonholonomic Motion Planning.
\end{IEEEkeywords}
	
\section{INTRODUCTION}
\label{sec:introduction}
We introduce a novel pursuit-evasion game closely related to mobile robotics applications.
In the literature, numerous pursuit-evasion games have been studied~\cite{CHUNG-11}. For example, one or more pursuers could be given the task of
finding an evader \cite{GUIBAS-99,ISLER-05,TOVAR-08} in an
environment. Another related problem is to maintain visibility
of a moving evader \cite{LAVALLE-97,HHG-02,JUNG-02,BHATTACHARYA-10, MURRIETA-11}. Alternatively, the pursuer might try to capture the evader by moving to a contact configuration or getting closer than a given distance~\cite{ISAACS-65,BHADAURIA-12,RUIZ-13,NOORI-14, STIFFLER-14}.

The kinematic problem of capturing an omnidirectional evader using a
Differential Drive Robot (DDR) in an obstacle-free environment was
studied in \cite{RUIZ-13}. In that work, it was assumed that the DDR
is faster than the evader, and the game ends when the distance between the
DDR and the evader is smaller than a critical value $l$. The DDR wants
to minimize the capture time while the evader wants to maximize
it. The main contributions of that work were computing time-optimal motion strategies
for each player using differential game theory \cite{ISAACS-65} and finding the conditions defining the winner. In contrast to \cite{RUIZ-13}, in this work we consider that the game takes place in a bounded environment and the players have the same maximum velocity. This setup requires a new solution and methodology.

Our game is a variant of the lion-and-man game in which a lion tries
to capture a man with equal maximum speed. This game has received
significant attention in robotics~\cite{CHUNG-11}. 
However, in most previous work on the lion-and-man game, the lion is assumed to be omnidirectional. This assumption is not true for most robots. Therefore, in this paper, we study a variant of the original lion-and-man game in which the pursuer is a differential drive.

In the original version of the game, the
lion and man are in a circular arena (see Fig. \ref{fig:figlionman}). They have the same maximum speed
and can observe each other at all times. The lion can get arbitrarily close to
the man using the following strategy first 
described in~\cite{LITTLEWOOD-53}: at the beginning of the game, the lion goes to the center  $C$ of the arena.
Afterwards, let $M'$ be the position of the man when it is the lion's turn. The lion moves on to the radius $CM'$. Among all points on $CM'$ within its step-size, the lion chooses the point $L'$ that is closest to $M'$. By using elementary trigonometry, it can be shown that the lion captures the man in $O(r^2)$ steps when they move in turns. Here, $r$ is the radius of the arena. 

Now imagine that both players are unit disks and the pursuer is a Differential Drive Robot (DDR). It can not directly follow the strategy described above. In particular, since it must spend time to turn, it can not reach $L'$. Can the pursuer capture the man? In this paper, we present a novel strategy and show that the DDR lion can still capture the man. 

We start by mapping the environment to the evader's configuration space. In the new environment both players are represented as points, and the configurations where the DDR collides with the evader are characterized by a capture disk of radius $r_c = 2$ centered at the DDR's position. This mapping allows us to focus on the differential constraints and ignore collisions with the boundary due to the shape of the players. Note that since the players have the same shape, the reachable part of the workspace is the same for both players. 
We propose a pursuit strategy to solve the game and prove that one DDR pursuer can capture (collide) the evader in upper bounded time. 

\begin{figure}[t]
\centering
\includegraphics[scale=0.35]{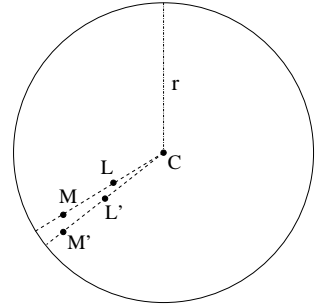}
\caption{The man moves from point $M$ to point $M'$. The lion at $L$ moves on to the radius $CM'$ and reaches point $L'$. \label{fig:figlionman}}
\end{figure}

\subsection{Related Work}
The lion and man game is well studied \cite{LITTLEWOOD-53, ALONSO-92,SGALL-01}. In \cite{LITTLEWOOD-53}, Littlewood shows that the lion can not reduce the distance to zero in the continuous time formulation of the problem. Alonso et al. \cite{ALONSO-92} showed that the lion captures the man in time $O(\frac{r}{s}\log\frac{r}{c})$, where $r$ is the radius of the circular arena, $c$ is the capture distance and $s$ is the maximum speed of the players. In \cite{SGALL-01}, Sgall studies the discrete time version of the problem in the positive quadrant. He showed under which initial conditions  the lion can capture the man.
Recently, it has been shown that a single lion can capture the man in simply-connected polygons~\cite{ISLER-05}, and three lions suffice in polygons with holes~\cite{BHADAURIA-12}.
In this work, we take a step toward modeling more realistic robotics applications and study the lion and man game where the pursuer is a  DDR that can only change its motion direction at a bounded rate that is inversely proportional to its translational speed~\cite{RUIZ-13}. Using a novel strategy, we study conditions under which the DDR lion can capture the man in any convex environment (see Fig.~\ref{fig:figworkspace}).

\begin{figure}[t]
\centering
\includegraphics[scale=0.35]{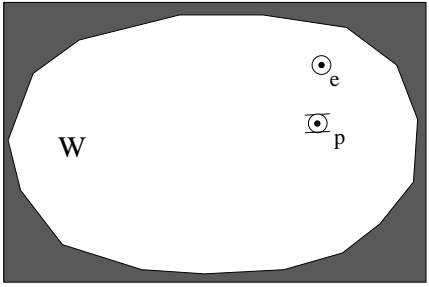}
\caption{The game takes place in a convex environment $W$. The DDR at position $p$ and the evader at position $e$ are unit disks. The DDR wins the game if it captures (collides with) the evader. The evader wins if it avoids capture forever. \label{fig:figworkspace}}
\end{figure}

\section{PROBLEM FORMULATION}
\label{sec:problemformulation}

A Differential Drive Robot pursuer (DDR) and an omnidirectional evader move in a convex environment $W \in \mathbb{R}^2$ (see Fig. \ref{fig:figworkspace}). The DDR wants to capture the evader, and the evader wants to avoid capture. The capture condition is satisfied when the DDR is in collision with the evader. The DDR wins the game if it captures the evader in finite time. The evader wins if it avoids capture forever. Both players are unit disks and have the same maximum bounded speed $v^{\max}$. The DDR can only change its motion direction at a bounded rate that is inversely proportional to its translational speed \cite{RUIZ-13}. In this work, we consider a purely kinematic problem, and neglect any effects due to dynamic constraints (e.g., acceleration bounds).  The motions of the players are made in turns, with the evader moving first. The duration of each turn is chosen to be $1/v^{\max}$ which is the time it takes to travel a unit distance. Our result holds for any non-zero turn duration with the appropriate scaling of the number of steps. We assume a complete information setup, i.e. each player knows the location and orientation of the other player at all times.

\begin{figure}[t]
\centering
\subfigure[Model]{
\label{fig:modelDDR}
\includegraphics[scale=0.42]{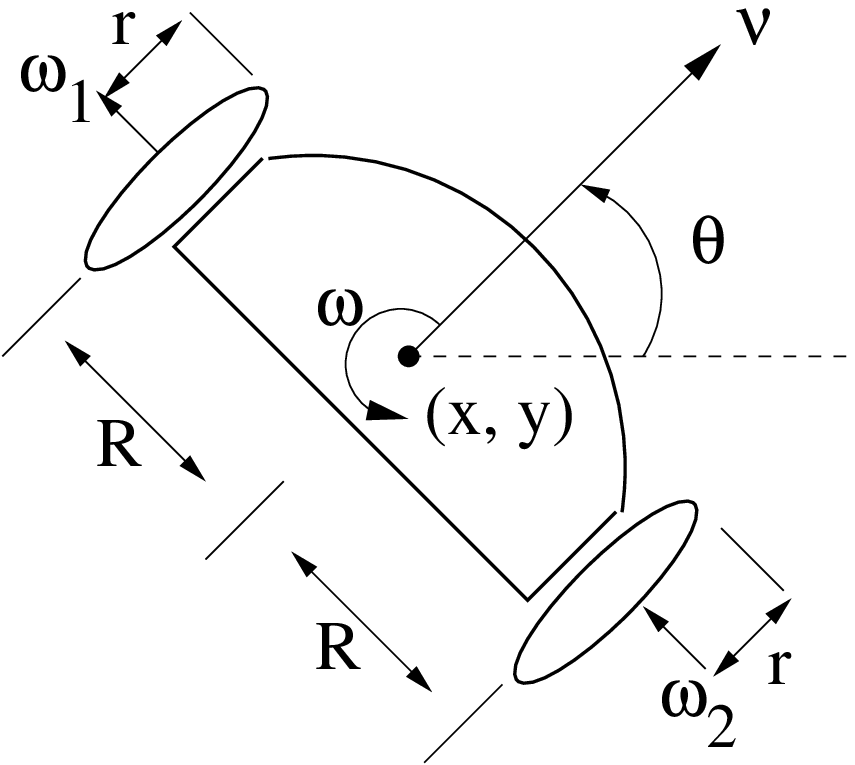}
}
\subfigure[Control space] {
\label{fig:perfilddr}
\includegraphics[scale=0.46]{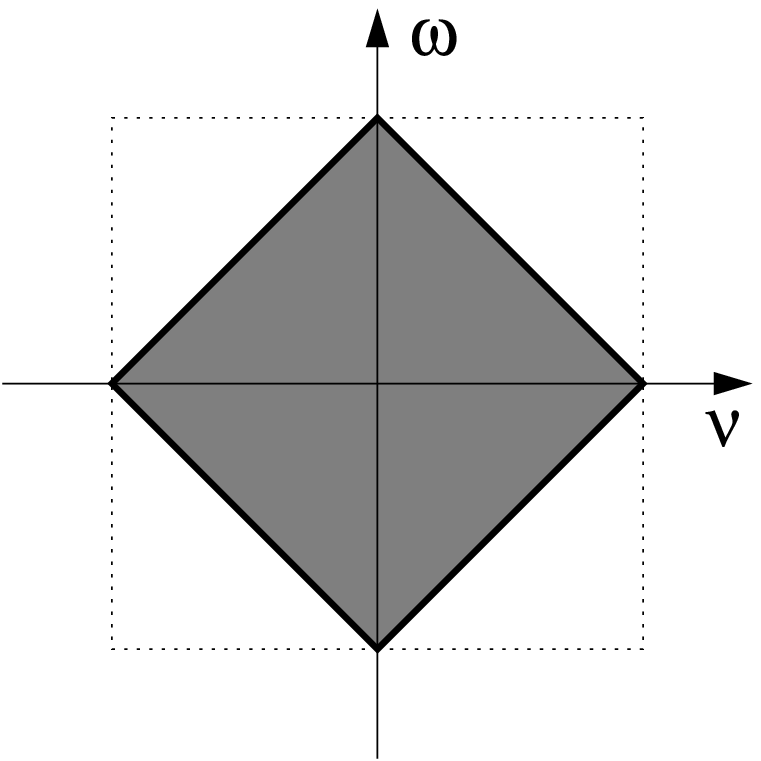}
}
\caption{Differential Drive Robot Model}
\end{figure}

\section{PRELIMINARIES}
\label{sec:preliminaries}

In this section, we present the concepts and definitions used throughout the paper. 

\subsection{Model}

The kinematic model for a DDR \cite{RUIZ-13} is given by 
\begin{equation}
\label{eq:ddr}
\begin{split}
\dot x_p = v \cos \theta_p,\:\:
\dot y_p = v \sin \theta_p,\:\:
\dot \theta_p = \omega
\end{split}
\end{equation}
where $v$ and $\omega$ are the translational and angular velocities of the DDR. In practice, we control the individual velocities of the wheels. Therefore, we have that

\begin{equation}
\label{eq:controlsdef}
\begin{split}
v = \frac{r \left( \omega_1+ \omega_2 \right)}{2},\:\:\:\:\: \omega = \frac{r \left( \omega_2 - \omega_1 \right)}{2R}
\end{split}
\end{equation}
where $r$ is the radius of the wheels, and $\omega_1$ and $\omega_2$ are their angular velocities. $R$ is the distance between the center of the robot and the wheel's location (see Fig. \ref{fig:modelDDR}). For a DDR, assuming  $v^{\max}>0$, we have that

\begin{equation}
\label{eq:vwddr}
|\dot \theta| = |\omega| \leq \frac{1}{R}(v^{\max}- |v|)
\end{equation}
The angular velocity is inversely proportional to the translation velocity (see Fig. \ref{fig:perfilddr}). We assume that the values of the translational velocity $v$  and the angular velocity $w$ can be chosen directly as long as they satisfy Eq. (\ref{eq:vwddr}).

\subsection{Playing space}

As  mentioned earlier, we map the convex environment $W$ to the evader's configuration space simply by removing all points within unit distance from the boundary.  
We denote this new environment as $Q$.  From now on, we represent the players as  points in $Q$.  Now we identify configurations where collisions between the DDR and the evader are possible in $W$. In Fig. \ref{fig:figcapturegion}, we observe that collisions in $W$ occur for the evader's positions that are at distance at most 2 to the DDR's center thus we can characterize those configurations in $Q$ using a circle of radius $r_c=2$ centered at $p$. We denote that circle as the capture region.
 
 Hereafter $p \in Q$ represents the position of the DDR, and $\theta_p$ denotes its orientation (heading) with respect to a line $L_G$ which will be explained shortly. The location of the evader is denoted by $e\in Q$. We assume that both players have the same maximum speed $v^{\max}=1$ and the DDR's wheels have radius one thus $r=1$. Each player moves during a time-step $\Delta t = 1$. The length of the shortest path between two points $a,b\in Q$ is denoted by $d(a,b)$. A chord, or a diagonal of $Q$ is a line segment joining two non-consecutive vertices of $Q$. The diameter of $Q$ is   the length of the longest chord of $Q$ and denoted by $diam(Q)$. The boundary of $Q$ is represented as $\partial Q$.

\begin{figure}[t]
\centering
\includegraphics[scale=0.62]{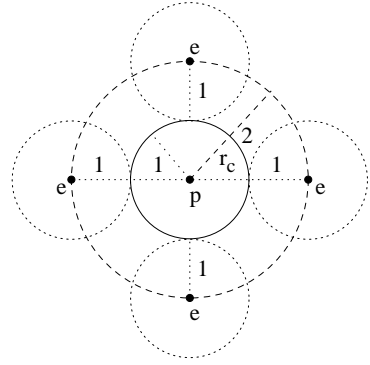}
\caption{The dashed circle with radius $r_c$ centered at the DDR's position $p$ characterizes the configurations where the DDR collides with the evader. Note that since the players are represented by unit disks we have that $r_c=2$. \label{fig:figcapturegion}}
\end{figure}

\subsection{Notion of guarding a line segment}

Next, we define a fundamental concept behind the notion of guarding a line segment.

\begin{definition}[Projection]
Let $L_G\in Q$ be a line segment dividing $Q$ into two subregions $Q_1$ and $Q_2$, and $e\in Q_2$. The {\em projection} of $e$ on $L_G$ is the closest point $e_{\pi}\in L_G$ to $e$.
\end{definition}

In this paper, we take advantage of the fact that the DDR has a non-zero capture radius. In Lemma \ref{lm:relocation}, we prove that the DDR can be located at a distance $d(p,e_\pi) = \frac{1}{2}$ on $L_G$, and it captures the evader if it tries to cross $L_G$.
	
\begin{definition}[Guarding a line]
The line segment $L_G$ is guarded if the DDR can prevent the evader from crossing it. That is, the evader is captured if it crosses $L_G$.
\end{definition}

In Lemma~\ref{lm:relocation}, we will show that the DDR can guard $L_G$ by maintaining the following invariants at the beginning of each turn of the evader.

\begin{invariant}
\label{inv1}
The heading of the DDR is {\em parallel} to $L_G$.
\end{invariant}

\begin{invariant}
\label{inv2} 
The DDR is located at a point $p\in L_G$ such that $d(p,e_\pi)\ = \frac{1}{2}$.
\end{invariant}

In Lemma \ref{lm:initial}, we prove that Invariants \ref{inv1} and \ref{inv2} can be established by the DDR in finite time.  Once the invariants have been established in the game, the DDR has to ensure that after performing its motions both invariants are restored before its turn ends.
\subsection{Local reference frame}
Suppose the DDR is guarding $L_G$. We define a local reference frame $R_L$ with its origin at the point $e_\pi$ and the $x$-axis aligned with $L_G$ (see Fig. \ref{fig:capturecondition}). The positive $y$-axis points towards the region containing the evader. If $e_\pi$ is located to the right side of the DDR then the positive $x$-axis is pointing right, otherwise the positive $x$-axis is pointing left. The frame $R_L$ is created at the beginning of each evader's turn taking the projection of the initial position of the evader as its origin. We use this frame to describe the player's motion. 

\begin{definition}[Positive projection]
The evader has a {\em positive projection} if after its turn it has a new projection $e'_\pi$ located on the positive side of the $x$-axis in $R_L$.
\end{definition}

\begin{definition}[Negative projection]
The evader has a {\em negative projection} if after its turn it has a new projection $e'_\pi$ located at the origin or on the negative side of the $x$-axis in $R_L$.
\end{definition}

\subsection{Notions of progress}

We introduce two notions of progress in our game. Those concepts will be used to prove that the DDR captures the evader in finite time.

\begin{definition}[Vertical progress]
The DDR makes {\em vertical progress} if it can guard a new line $L'_G$ which is parallel to $L_G$ and closer to the evader. 
\end{definition}

When the DDR makes vertical progress, it reduces the  size of the region $Q_2$ containing the evader.
\begin{definition}[Horizontal progress]
The DDR makes {\em horizontal progress} if it increases its $x$ coordinate in $R_L$ and guards $L_G$.
\end{definition}

When the DDR makes horizontal progress, it pushes the evader towards $\partial Q$. Note that in the previous definition it is assumed that at the beginning of the DDR's turn the evader has a positive projection.

\begin{figure}[t]
\centering
\includegraphics[scale=0.35]{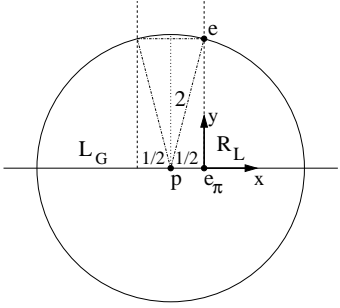}
\caption{We define a local reference frame $R_L$ at $e_\pi$. In this case, the $x$-axis is pointing right. \label{fig:capturecondition}}
\end{figure}

\section{THE CAPTURE STRATEGY}
\label{sec:estrategy}

In this section, we introduce the strategy for capturing the evader.

We divide the strategy into two stages.
In the first one, the DDR moves to a location on a longest chord of $Q$, and it establishes Invariants \ref{inv1} and \ref{inv2}. After that  both invariants are maintained throughout the game. In the second stage, the DDR makes vertical progress reducing the subregion containing the evader until the capture is achieved (see Fig. \ref{fig:figpol1}). For the first stage, in Lemma \ref{lm:initial} we show that once the DDR is located on a longest chord of $Q$, it can establish Invariants \ref{inv1} and \ref{inv2} in a finite number of turns. For the second stage, in Lemma \ref{lm:moveup} we prove that the DDR makes vertical progress if after the evader's turn one of the following conditions holds: 1)~the evader has a negative projection or 2)~the evader has a positive projection and the distance between its previous and current projections on $L_G$ is less than a threshold value. If the previous conditions do not hold then the DDR makes horizontal progress. In that case, in Lemma \ref{lm:horizontalprogress} we prove that after an upper bounded number of steps the evader hits $\partial Q$ and the DDR can make vertical progress. In Theorem \ref{tm:capture}, we show that using this strategy the evader is captured in an upper bounded time.

\begin{figure}[t]
\centering
\includegraphics[scale=0.35]{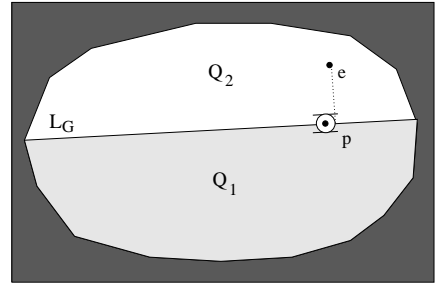}
\caption{Let $Q$ be the environment obtained from mapping the workspace $W$ into the evader's configuration space. The DDR at point $p\in Q$ guards the line segment $L_G\in Q$ by maintaining a distance $\frac{1}{2}$ to the projection $e_\pi$ of the evader's position $e$ on the line $L_G$. $L_G$ divides $Q$ into two subregions $Q_1$ (light gray) and $Q_2$ (white). In this case, $e\in Q_2$. If the evader crosses the line $L_G$ it will be captured by the DDR. The DDR moves the line $L_G$ towards the evader's position $e$ at some stages of the game, increasing the area of the subregion $Q_1$, until capture is attained. \label{fig:figpol1}}
\end{figure}

\subsection{Guarding a line segment}
\label{sec:guarding}

Let $L$ be a longest chord of $Q$. 
In this subsection, we first prove that the DDR can establish Invariants $\ref{inv1}$ and $\ref{inv2}$ on $L$.

\begin{lemma}
\label{lm:initial}
Let $L$ be a longest chord of $Q$. 
The DDR can establish Invariants \ref{inv1} and \ref{inv2} on $L$ in $\lceil diam(Q) \rceil$  steps. 
Afterwards, the DDR can maintain both invariants indefinitely by following the projection $e_\pi$ of $e$ on $L$.
\end{lemma}

\begin{proof}
Suppose the DDR is located at an arbitrary point $p$ on $L$ and it has aligned its heading with $L$ (Invariant \ref{inv1}). 
Let $e$ and $e'$ be the evader's position before and after its turn, respectively. 
We know that $d(e,e')\leq 1$. 
Since $Q$ is convex, $e_\pi$ and $e'_\pi$ are the perpendicular projections of $e$ and $e'$ on to $L$, thus $d(e_\pi,e'_\pi)\leq d(e,e')$. 
The DDR establishes the invariants as follows. In each turn, the DDR  moves from point $p$ to point $p'$, both on $L$, such that $d(p,p')\leq 1$. The point $p'$ is chosen according to the following rule: 
If $d(p,e'_\pi)>\frac{3}{2}$, the DDR moves to the point $p'$ on $L$ such that $d(p,p')=1$ and $d(p',e'_\pi)<d(p,e'_\pi)$. Otherwise, the DDR moves to the point $p'$ such that $d(p',e'_\pi)=\frac{1}{2}$ and $d(p,p')\leq 1$.   

Note that the DDR either takes a full step or catches up with $e'_\pi$. 
Further, if the evader keeps moving in one direction pushing its projection away from the DDR and forcing it to take a full step, it will hit the boundary first since the DDR is on the longest chord and behind the projection. Therefore, the projection of the evader must cross the DDR before the DDR hits the boundary. At this point, the DDR can position itself half a step behind the projection.
Therefore, following this strategy the DDR will position itself on a point $p$ on $L$ such that $d(p,e_\pi)=\frac{1}{2}$ in at most $\lceil diam(Q) \rceil$ steps. 
\end{proof}

While the DDR is establishing Invariant \ref{inv2}, the evader can cross $L$ multiple times. This is not a problem: Since $Q$ is convex,  every point in $Q$ has a unique projection on the longest chord. Furthermore, the projection moves continuously. Thus both invariants can be established regardless of which side the evader lies. Once the invariants are established on any line segment $L_G$, they can be easily maintained. This is because the evader's projection moves by at most one unit which allows the pursuer to keep up with the motion of the projection.

\begin{corollary}
Let $L_G$ be any chord of $Q$ and suppose that the DDR established Invariants \ref{inv1} and \ref{inv2} on $L_G$. 
The DDR can maintain both invariants indefinitely.
\end{corollary}

The goal of establishing the invariants is trapping the evader in one of the two regions defined by longest chord in $Q$. This region is the one which contains the evader when the invariants are established and is arbitrary.

We now prove that by maintaining Invariants $\ref{inv1}$ and $\ref{inv2}$, the pursuer can prevent the evader from crossing the line segment $L_G$ . This introduces a notion of guarding the line segment which is an important component of the pursuer strategy.

\begin{lemma}
\label{lm:relocation}
After Invariants \ref{inv1} and \ref{inv2} are established, the evader cannot cross the line $L_G$: if the evader moves to a new position $e'$ such that $d(e',e'_\pi) < \frac{\sqrt{15}}{2}$ then the DDR captures the evader in one turn.
\end{lemma}
\begin{proof}
Let $p$ be the position of the DDR on $L_G$. As Invariant \ref{inv2} has been established the evader is located on a point $e$ such that $d(p,e_\pi)=\frac{1}{2}$. 
Consider the point $e^*$  on the boundary of the capture region (see Fig. \ref{fig:evadercross}). Since $Q$ is convex the segment $\overline{e^*e_\pi}$ is the shortest path from $e^*$ to $L_G$. The length of this segment is $d(e^*, e_\pi)=\sqrt{2^2-(1/2)^2}=\sqrt{15}/2$. 
 Therefore, for all positions $e$ located outside the capture region and  such that $d(p,e_\pi)=\frac{1}{2}$ the distance $d(e,e_\pi)\geq \sqrt{15}/2 > 1$, thus the evader cannot reach the line $L_G$ without entering the capture region. 
Note that as $\overline{ee_\pi}$ is the shortest path from $e$ to $e_\pi$ on $L_G$ then any other trajectory between those points is longer than $d(e,e_\pi)$ and it cannot be traveled in one turn. 
If, instead of crossing the line, the evader moves closer to it without crossing, it still gets captured: For any evader move to $e'$ such that $d(e',e'_\pi) < \sqrt{15}/2$, the DDR can move to location $p'$ such that $d(p',e'_\pi)=\frac{1}{2}$ and capture the evader since $d(p',e')=\sqrt{(1/2)^2+d(e',e'_\pi)^2} < 2$.
\end{proof}

\begin{figure}[t]
\centering
\includegraphics[scale=0.35]{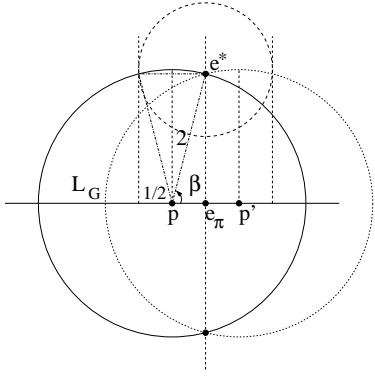}
\caption{The evader is located at point $e^*$ on the boundary of the capture region. The dashed circle shows all evader's locations where $d(e,e')=1$. The DDR can move from point $p$ to point $p'$ where $d(p,p')\leq 1$ in one turn.
\label{fig:evadercross}}
\end{figure}

\subsection{Making progress}

In the previous section, we showed that the DDR can establish Invariants \ref{inv1} and \ref{inv2}  on the longest chord and guard any chord once the invariants have been established. Of course, simply guarding does not suffice for capture. 
In this subsection, we present a strategy to guard a new line parallel  to $L_G$ after Invariants \ref{inv1} and \ref{inv2} are established. We also establish lower bounds for the horizontal and vertical progress made by the DDR at each turn. Those bounds are used to prove that the DDR captures the evader in an upper bounded time.

Suppose the evader moves from point $e$ to point $e'$ such that $d(e,e')\leq 1$ and its new projection $e'_\pi$ is negative. The DDR makes vertical progress by moving to point $p'$ in the line perpendicular to $L_G$ having the point $p$ as the intersection of both lines, and $d(p',e')<d(p,e)$ (see Fig. \ref{fig:bordermove}). To reach $p'$ we propose a strategy consisting of four motions that need to be performed in one turn. First, the DDR rotates in place an angle $\alpha$ at point $p$. After that, the DDR translates a distance $d(p,p_t)$ from point $p$ to point $p_t$ (see Fig. \ref{fig:bordermove}). Next, the DDR rotates in place an angle $-\alpha$ at point $p_t$. Finally, the DDR translates a distance $d(p,p_t) \cos \alpha$ from point $p_t$ to point $p'$. We call this set of motions the {\em zig-zag strategy} (see Fig. \ref{fig:bordermove}). Note that after applying the zig-zag strategy the DDR has established Invariant \ref{inv1}. 

If $d(e_\pi,e'_\pi)=0$ or $d(e_\pi,e'_\pi)=1$ then at point $p'$ the DDR has also established Invariant \ref{inv2}. If $d(e_\pi,e'_\pi)$ is different from the previous values then the DDR has to perform an additional translation to establish Invariant \ref{inv2} (see Fig. \ref{fig:bordermove}). In this case, it has to move to a point $p''$ such that $d(p'', e'_\pi)=\frac{1}{2}$. To reach the point $p''$ the DDR translates a distance $\frac{1}{2}-d(p',e'_\pi)$.

If the evader has a positive projection and $d(e_\pi, e'_\pi)<1$ then the DDR can also make vertical progress by applying the zig-zag strategy reestablishing Invariant \ref{inv1}. In this case, it has to move to a point $p''$ such that $d(p'', e'_\pi)=\frac{1}{2}$. To reach the point $p''$ the DDR translates a distance $d(e_\pi, e'_\pi)$ (see Fig. \ref{fig:evadercase1}). Note that if $d(e_\pi, e'_\pi) = 1$ then the DDR can only follow the evader's projection in order to maintain Invariants \ref{inv1} and \ref{inv2}.
\begin{figure}
\centering
\includegraphics[scale=0.45]{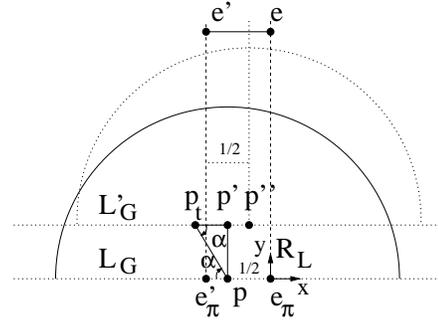}
\caption{Making vertical progress. The evader has a negative projection. The DDR makes vertical progress applying the zig-zag strategy. \label{fig:bordermove}}
\end{figure}
\begin{figure}
\centering
\includegraphics[scale=0.45]{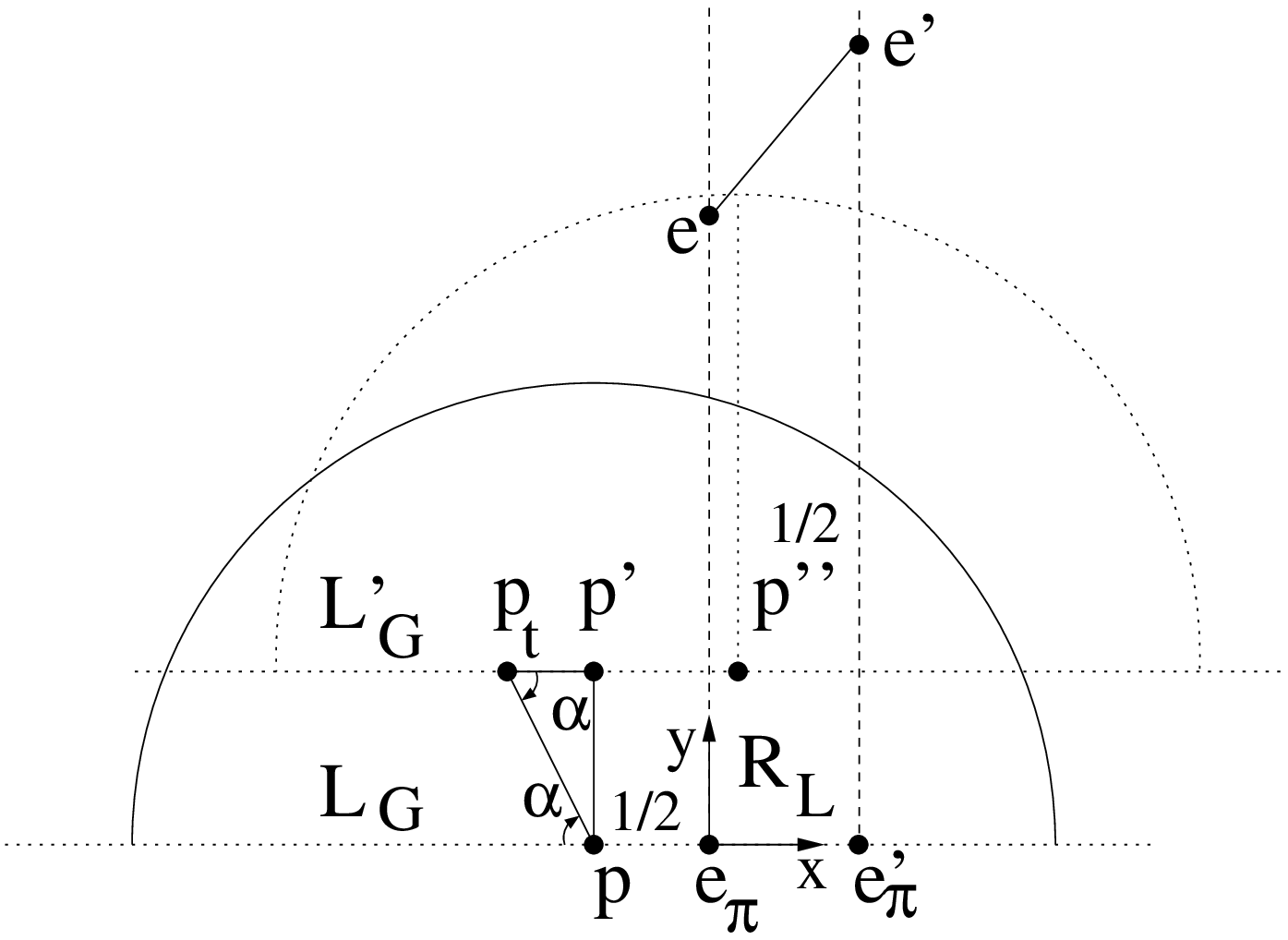}
\caption{Making horizontal progress. The evader has a positive projection. The DDR makes vertical progress using the zig-zag strategy if $d(e_\pi,e'_\pi) < 0.056$ otherwise it makes horizontal progress. \label{fig:evadercase1}}
\end{figure}
\begin{definition}[Bounding vertical progress]
We denote as $k_v$ the lower bound for the vertical progress made by the DDR when it applies the zig-zag strategy during its turn.
\end{definition}
\begin{definition}[Bounding horizontal progress]
We denote as $k_h$ the lower bound for the horizontal progress made by the DDR when it follows a positive projection of the evader.
\end{definition}

\begin{lemma}
\label{lm:moveup}
The DDR makes vertical progress of at least $k_v=0.0156$ using the zig-zag strategy if during the evader's turn one of the following conditions holds: 1) the evader has a negative projection or 2) the evader has a positive projection and the distance between its previous and current projections on $L_G$ is less than $k_h=0.056$.
\end{lemma}
\begin{proof}
Suppose the evader has a negative projection. Using the zig-zag strategy the DDR reaches a new position $p'$ on a parallel line to $L_G$ and it also reestablishes Invariant \ref{inv1} (see Fig. \ref{fig:bordermove}). In order to maintain Invariant \ref{inv2}, the DDR also needs to translate a distance $\frac{1}{2}-d(p',e'_\pi)$.
The time to perform the strategy described before is given by the following expression
\begin{equation}
\label{eq:t_s}
t_s=\frac{2\alpha}{v^{\max}}+\frac{d(p_t,p)(1+\cos \alpha)}{v^{\max}}+\frac{\left(\frac{1}{2}-d(p',e'_\pi)\right)}{v^{\max}}
\end{equation}
Note that in order to satisfy the inverse relation between the rotational and translational velocities in Eq. (\ref{eq:vwddr}), we assume that the rotational and translational motions are performed separately at maximal speed in Eq. (\ref{eq:t_s}), thus each motion requires an independent amount of time.
Since the DDR must perform the strategy in one turn 
$t_s=\Delta t=1$.
Substituting the last equation into Eq. (\ref{eq:t_s}), we obtain that
\begin{equation}
\label{eq:d(ptp)}
d(p,p_t)=\frac{\frac{1}{2}+d(p',e'_\pi)-2\alpha}{1+\cos\alpha}
\end{equation}
From the right triangle $\triangle pp_tp'$ in Fig. \ref{fig:bordermove}, we have that $d(p,p')=d(p,p_t)\sin \alpha$.
Substituting Eq. (\ref{eq:d(ptp)}) into the last expression, and recalling that $\tan (\frac{\alpha}{2})=\frac{\sin \alpha}{1+\cos \alpha}$ we obtain an expression for the distance that the DDR can move perpendicular to $L_G$ and allows the DDR to reestablish Invariants \ref{inv1} and \ref{inv2} 
\begin{equation}
\label{eq:updistance}
d(p,p')=\left(\frac{1}{2}+d(p',e'_\pi)-2\alpha\right)\tan\left(\frac{\alpha}{2}\right)
\end{equation}
Eq. (\ref{eq:updistance}) depends on the angle $\alpha$ that the DDR initially rotates and the distance between the DDR's position $p'$ and the projection $e'_\pi$, then to find the minimum vertical progress that the DDR can make we have to do the following. For each value of $d(p',e'_\pi)$ we have to find the value of $\alpha$ that maximizes Eq. (\ref{eq:updistance}), this give us a set of triplets $(d(p',e'_\pi), \alpha, d(p,p'))$ where each $d(p,p')$ is the maximum vertical progress made by the DDR for the corresponding value of $d(p',e'_\pi)$. The triplet having the smallest value of $d(p,p')$ corresponds to the minimum vertical progress made by the DDR when the evader has a negative projection. Note that $d(p',e'_\pi)\in[0,\frac{1}{2}]$ and $\alpha\in [0, \frac{\pi}{2}]$ since the DDR can rotate in both clockwise or counter-clockwise direction. As the value of $d(p',e'_\pi)$ decreases the time to establish Invariant \ref{inv2} increases, recall that the DDR has to translate a distance $\frac{1}{2}-d(p',e'_\pi)$, and since the time-step is fixed the time to perform the zig-zag strategy is reduced. It is not hard to see that the minimal vertical progress achieved by the zig-zag strategy when Eq. (\ref{eq:updistance}) is maximized corresponds to $d(p',e'_\pi)=0$ since this requires the maximum translation to establish Invariant \ref{inv2}. To find the value of $\alpha$ that maximizes Eq. (\ref{eq:updistance}) when $d(p',e'_\pi)=0$ we need to solve
\begin{equation}
\label{eq:firstpartial}
\begin{split}
\frac{\partial d(p,p')}{\partial \alpha}=\frac{\partial \left[ \left(\frac{1}{2}-2\alpha\right)\tan\left(\frac{\alpha}{2}\right)\right]}{\partial \alpha}=0\\
\end{split}
\end{equation}
and verify that 
\begin{equation}
\label{eq:secondpartial}
\frac{\partial^2 d(p,p')}{\partial^2 \alpha}=\frac{\partial^2 \left[ \left(\frac{1}{2}-2\alpha\right)\tan\left(\frac{\alpha}{2}\right)\right]}{\partial^2 \alpha} < 0
\end{equation}
From Eq. (\ref{eq:firstpartial}) we have that
\begin{equation}
\left(\frac{1}{4} - \alpha \right)\sec \left(\frac{\alpha}{2}\right)^2-2\tan \left( \frac{\alpha}{2} \right) = 0
\end{equation}
This nonlinear equation can be solved using a numerical software package. We have that $\alpha=0.1251$ is an approximate solution for this equation. From Eq. (\ref{eq:secondpartial}) we have that 
\begin{equation}
\label{eq:secondpartial2}
\frac{\partial^2 d(p,p')}{\partial^2 \alpha} = -2\sec\left( \frac{\alpha}{2} \right)^2+\left( \frac{1}{4} - \alpha \right)\sec \left( \frac{\alpha}{2}\right)^2 \tan \left( \frac{\alpha}{2} \right)
\end{equation}
Substituting $\alpha = 0.1251$ into Eq. (\ref{eq:secondpartial2}) we have that $\frac{\partial^2 d(p,p')}{\partial^2 \alpha}=-1.9999 < 0$, thus $\alpha=0.1251$ maximizes $d(p,p')$ when $d(p',e'_\pi)=0$. Substituting $\alpha=0.1251$ into Eq. (\ref{eq:updistance}) we obtain $d(p,p')=0.0156$. This value corresponds to the minimum vertical progress made by the DDR thus $k_v=0.0156$.

In an analogous way, if the evader has a positive projection but the projection is within a bound $k_h$, we found an expression for the distance that the DDR can move perpendicular to $L_G$ and allows the DDR to reestablish Invariants \ref{inv1} and \ref{inv2} (see Fig. \ref{fig:evadercase1}) 
\begin{equation}
\label{eq:updistance2}
d(p,p')=\left(1 - d(e_\pi,e'_\pi)-2\alpha\right)\tan\left(\frac{\alpha}{2}\right)
\end{equation}
Note that Eq. (\ref{eq:updistance2}) depends on the distance between the previous and current projections of the evader's position on $L_G$. In this case, we obtain the value when the horizontal progress made by the DDR simply following the projection $e'_\pi$ is the same than the vertical progress made by the DDR applying the zig-zag strategy. We use this value as the lower bound for the minimal horizontal progress $k_h$. To find it we do the following. For each value of $d(e_\pi,e'_\pi)$ we have to find the value of $\alpha$ that maximizes Eq. (\ref{eq:updistance2}), this give us a set of triplets $(d(e_\pi,e'_\pi), \alpha, d(p,p'))$ where each $d(p,p')$ is the maximum vertical progress made by the DDR for the corresponding value of $d(e_\pi,e'_\pi)$. Note that the value of $\alpha$ that maximizes Eq. (\ref{eq:updistance2}) can be computed as it was described in the first part of the proof. We select the triplet satisfying $d(p,p')=d(e_\pi, e'_\pi)$ and we set $k_h=d(e_\pi,e'_\pi)$. This value is $d(e_\pi,e'_\pi)=0.056$ for which the DDR performs a rotation $\alpha=0.2371$. We use $k_h$ as a {\em threshold}, if the evader has a positive projection  with $d(e_\pi,e'_\pi) < k_h$, the DDR applies the zig-zag strategy and makes vertical progress moving at least a perpendicular distance of $0.056$ to $L_G$, otherwise, it moves on $L_G$ to maintain Invariant \ref{inv2}, making horizontal progress by translating at least a distance of $0.056$.
\end{proof}
	
\begin{lemma}
\label{lm:horizontalprogress}
The DDR makes vertical progress after $O\left(\frac{diam(Q)}{k_h}\right)$ horizontal steps where $k_h=0.056$.
\end{lemma}
\begin{proof}
Let the DDR guards a line segment $L_G$. Suppose that after each turn the evader has a positive projection and the distance between its previous and current projections on $L_G$ is at least $k_h$, otherwise, the DDR makes vertical progress by Lemma \ref{lm:moveup}. Since the length of $L_G$ is upper bounded by $diam(Q)$ then in at most $O(diam(Q)/k_h)$ steps the projection of the evader reaches the boundary of $L_G$. After that the evader cannot continue having a positive projection and the DDR makes vertical progress. Note that since $Q$ is convex, the DDR starts guarding a longest chord of $Q$, and every new line segment $L_G$ guarded by the DDR is parallel to the previous one then the line guarded by the DDR before applying the zig-zag strategy is the longest line segment in the subregion $Q_2$ containing the evader, therefore the DDR never reaches $\partial Q$ before the evader.
\end{proof}

\subsection{Capturing the evader}

In our main theorem, we prove that following the strategy described in previous subsections the DDR captures the evader in an upper bounded time.

\begin{theorem}
\label{tm:capture}
After Invariants \ref{inv1} and \ref{inv2} have been established, the DDR captures the evader in any convex environment in $O\left(\frac{diam(Q)^2}{k_hk_v}\right)$ steps where $k_v=0.0156$ and $k_h=0.056$.
\end{theorem}

\begin{proof}
From Lemma \ref{lm:horizontalprogress}, the DDR makes vertical progress in at most $O(diam(Q)/0.056)$ horizontal steps for any line segment $L_G$. 
From Lemma \ref{lm:moveup}, the DDR minimal vertical progress is $k_v=0.0156$, thus in at most $O(diam(Q)/0.0156)$ vertical steps the DDR captures the evader by trapping it between the boundary and $L_G$. The result follows.
\end{proof}

Throughout the paper, we assumed that the duration of each time-step is $1/v^{\max}$. If a different time-step $\epsilon/v^{\max}$ is chosen, the number of steps to capture the evader gets scaled by $1/\epsilon^2.$
\section{CONCLUSIONS AND FUTURE WORK}

In this paper, we studied the kinematic problem of capturing an omnidirectional evader using a Differential Drive Robot (DDR) in convex
environments. The DDR wins the game if at any time instant it collides with the evader. We proved that for any convex environment if both players have the same maximum speed then a single DDR can capture the evader in $O\left(\frac{diam(Q)^2}{k_vk_h}\right)$ steps where $k_v=0.0156$ and $k_h=0.056$. Our result is one of the few closed form solutions of a differential game in a geometric setting.  

One avenue for future research is to improve the capture time: From Alonso et al. \cite{ALONSO-92} we have that an omnidirectional lion captures the man in time $O(\frac{r}{s}\log\frac{r}{c})$, where $r$ is the radius of the circular arena, $c$ is the capture distance and $s$ is the maximum speed of the players. Hence there might be room for improvement in capture time. The second avenue for research is to investigate the class of environments in which the DDR wins the game.
We conjecture that the DDR lion can win the game in any simply-connected domain. 

In practice, following the strategy presented in this paper might be difficult for some systems with dynamic constrains. 
In particular, following the zig-zag strategy may require large acceleration. 
Still, the pursuit strategy presented in this paper can be modified for some practical applications. For example, if the robot can not follow the zig-zag trajectory in a single time step, the duration of a step can be increased. In this case, our algorithm would guarantee capture within distance traveled in a single time step. Also, the robot does not need to perform the exact zig-zag motion, it just needs to arrive at the same configuration, i.e., reach a position where it makes progress and maintains Invariants 1 and 2. An interesting extension could be finding a smoother trajectory to reach that point using an optimization technique. Additional alternatives are increasing the capture radius or models where the DDR is faster than the evader. Studying systems with general dynamics is an interesting and challenging avenue for future research.

\section{ACKNOWLEDGMENTS }

This material is based in part upon work supported by the 
National Science Foundation under Grant Numbers IIS-1317788, IIS-1111638  and IIS-0917676.
Any opinions, findings, and conclusions or recommendations 
expressed in this material are those of the author(s) and do not necessarily
 reflect the views of the National Science Foundation. Most of this work was done when Ubaldo Ruiz was a postdoctoral fellow in the Department of Computer Science and Engineering of the University of Minnesota. He was financially supported by Consejo Nacional de Ciencia y Tecnolog\'ia (CONACYT), M\'exico.

\end{document}